\documentclass{article} 
\usepackage{fourier}
\usepackage{amsmath,amssymb,amsthm}
\usepackage{fullpage}
\usepackage{url}
\usepackage[numbers]{natbib}
\usepackage{color}
\usepackage{graphicx}

\usepackage{algorithm,algorithmic}
\usepackage[colorlinks=true,linkcolor=blue]{hyperref}

\newtheorem{theorem}{Theorem}
\newtheorem{lemma}[theorem]{Lemma}

\newtheorem{corollary}[theorem]{Corollary}

\newtheorem{remark}{Remark}

\theoremstyle{definition}
\newtheorem{definition}{Definition}

\let\inf\undef
\DeclareMathOperator*{\inf}{\vphantom{p}inf}
\let\sup\undef
\DeclareMathOperator*{\sup}{\vphantom{p}sup}

\newcommand{\mbb}[1]{\mathbb{#1}}

\newcommand{\mc}[1]{\mathcal{#1}}
\newcommand{\mrm}[1]{\mathrm{#1}}

\newcommand{\argmin}[1]{\underset{#1}{\mrm{argmin}} \ }

\newcommand{\reals}{\mathbb{R}}
\newcommand{\En}{\mathbb{E}}  

\renewcommand{\Pr}[1]{\mathbb{P}\left(#1\right)}

\newcommand{\tr}{\ensuremath{{\scriptscriptstyle\mathsf{T}}}}

\newcommand\cN{\mathcal{N}}

\newcommand\X{\mathcal{X}}

\newcommand\Z{\mathcal{Z}}
\newcommand\F{\mathcal{F}}
\newcommand\G{\mathcal{G}}
\newcommand\cH{\mathcal{H}}

\newcommand\Rad{\mathfrak{R}}

\def\deq{\triangleq}

\newcommand{\excessloss}{\mathcal{E}}

\newcommand{\emp}{\widehat{\En}}
\newcommand{\pred}{\widehat{f}}

\renewcommand{\star}{\text{star}}

\renewcommand{\H}{\mathcal{H}}
\newcommand{\B}{\mathcal{B}}

\renewcommand{\S}{\mathcal{S}}
\newcommand{\C}{\mathcal{C}}

\title{Learning with Square Loss: Localization through Offset Rademacher Complexity}
\date{}
\author{Tengyuan Liang \thanks{Department of Statistics, The Wharton School, University of Pennsylvania} \and Alexander Rakhlin \footnotemark[1] \and Karthik Sridharan \thanks{Department of Computer Science, Cornell University}}

\begin{document}

\maketitle

\begin{abstract}
	We consider regression with square loss and general classes of functions without the boundedness assumption. We introduce a notion of offset Rademacher complexity that provides a transparent way to study localization both in expectation and in high probability. For any (possibly non-convex) class, the excess loss of a two-step estimator is shown to be upper bounded by this offset complexity through a novel geometric inequality. In the convex case, the estimator reduces to an empirical risk minimizer. The method recovers the results of \citep{RakSriTsy15} for the bounded case while also providing guarantees without the boundedness assumption. 
\end{abstract}

\section{Introduction}

Determining the finite-sample behavior of risk in the problem of regression is arguably one of the most basic problems of Learning Theory and Statistics. This behavior can be studied in substantial generality with the tools of empirical process theory. When functions in a given convex class are uniformly bounded, one may verify the so-called ``Bernstein condition.'' The condition---which relates the variance of the increments of the empirical process to their expectation---implies a certain localization phenomenon around the optimum and forms the basis of the analysis via \emph{local Rademacher complexities}. The technique has been developed in \citep{KolPan00,koltchinskii2011oracle,bousquet2002some,bartlett2005local,bousquet2002concentration}, among others, based on Talagrand's celebrated concentration inequality for the supremum of an empirical process.

In a recent pathbreaking paper, \cite{Mendelson14} showed that a large part of this  heavy machinery is not necessary for obtaining tight upper bounds on excess loss, even---and especially---if functions are unbounded. Mendelson observed that only one-sided control of the tail  is required in the deviation inequality, and, thankfully, it is the tail that can be controlled under very mild assumptions.

In a parallel line of work, the search within the online learning setting for an analogue of  ``localization'' has led to a notion of an ``offset'' Rademacher process \citep{RakSri14nonparam}, yielding---in a rather clean manner---optimal rates for minimax regret in online supervised learning. It was also shown that the supremum of the offset process is a lower bound on the minimax value, thus  establishing its intrinsic nature. The present paper blends the ideas of \cite{Mendelson14} and \cite{RakSri14nonparam}. We introduce the notion of an offset Rademacher process for i.i.d. data and show that the supremum of this process upper bounds (both in expectation and in high probability) the excess risk of an empirical risk minimizer (for convex classes) and a two-step Star estimator of \cite{audibert2007progressive} (for arbitrary classes). The statement holds under a weak assumption even if functions are not uniformly bounded.


The offset Rademacher complexity provides an intuitive alternative to the machinery of local Rademacher averages. Let us recall that the Rademacher process indexed by a function class $\G\subseteq \reals^\X$ is defined as a stochastic process $g\mapsto \frac{1}{n}\sum_{t=1}^n \epsilon_t g(x_t)$ where $x_1,\ldots,x_n \in\X$ are held fixed and $\epsilon_1,\ldots,\epsilon_n$ are i.i.d. Rademacher random variables. We define the offset Rademacher process as a stochastic process
$$g\mapsto \frac{1}{n}\sum_{t=1}^n \epsilon_t g(x_t) - c g(x_t)^2$$
for some $c\geq 0$. The process itself captures the notion of localization: when $g$ is large in magnitude, the negative quadratic term acts as a compensator and ``extinguishes'' the fluctuations of the term involving Rademacher variables. The supremum of the process will be termed \emph{offset Rademacher complexity}, and one may expect that this complexity is of a smaller order than the classical Rademacher averages (which, without localization, cannot be better than the rate of  $n^{-1/2}$). 

The self-modulating property of the offset complexity can be illustrated on the canonical example of a linear class $\G = \{x\mapsto w^\tr x: w\in \reals^p\}$, in which case the offset Rademacher complexity becomes
$$ \frac{1}{n}\sup_{w\in\reals^p} \left\{ w^\tr \left(\sum_{t=1}^n \epsilon_t x_t\right) - c \|w\|_\Sigma^2 \right\} = \frac{1}{4cn} \left\|\sum_{t=1}^n \epsilon_t x_t\right\|^2_{\Sigma^{-1}}$$
where $\Sigma=\sum_{t=1}^n x_tx_t^\tr$. Under mild conditions, the above expression is of the order $\mathcal{O}\left(p/n\right)$ in expectation and in high probability --- a familiar rate achieved by the ordinary least squares, at least in the case of a well-specified model. We refer to Section~\ref{sec:examples} for the precise statement for both well-specified and misspecified case.

Our contributions can be summarized as follows. First, we show that offset Rademacher complexity is an upper bound on excess loss of the proposed estimator, both in expectation and in deviation. We then extend the chaining technique to quantify the behavior of the supremum of the offset process in terms of covering numbers. By doing so, we recover the rates of aggregation established in \citep{RakSriTsy15} and, unlike the latter paper, the present method does not require boundedness (of the noise and functions). We provide a lower bound on minimax excess loss in terms of offset Rademacher complexity, indicating its intrinsic nature for the problems of regression. While our in-expectation results for bounded functions do not require any assumptions, the high probability statements rest on a lower isometry assumption that holds, for instance, for subgaussian classes. We show that offset Rademacher complexity can be further upper bounded by the fixed-point complexities defined by Mendelson \cite{Mendelson14}. We conclude with the analysis of ordinary least squares.


\section{Problem Description and the Estimator}

Let $\F$ be a class of functions on a probability space $(\X,P_X)$. The response is given by an unknown random variable $Y$, distributed jointly with $X$ according to $P=P_X \times P_{Y|X}$. We observe a sample $(X_1,Y_1),\ldots,(X_n,Y_n)$ distributed i.i.d. according to $P$ and aim to construct an estimator $\pred$ with small excess loss $\excessloss(\pred)$, where
\begin{align}
	\excessloss(g) ~\deq~ \En (g-Y)^2 - \inf_{f\in\F} \En(f-Y)^2
\end{align}
and $\En(f-Y)^2 = \En(f(X)-Y)^2$ is the expectation with respect to $(X,Y)$. Let $\emp$ denote the empirical expectation operator and define the following two-step procedure:
\begin{align}
	\label{eq:def_estimator}
	\widehat{g}=\argmin{f\in\F} \emp(f(X)-Y)^2, ~~~~ \pred=\argmin{f\in \star(\F,\widehat{g})} \emp(f(X)-Y)^2
\end{align}
where $\star(\F,g)=\{\lambda g+(1-\lambda)f: f\in\F, \lambda\in[0,1]\}$ is the star hull of $\F$ around $g$. (we abbreviate $\star(\F,0)$ as $\star(\F)$.) This two-step estimator was introduced (to the best of our knowledge) by  \cite{audibert2007progressive} for a finite class $\F$. We will refer to the procedure as the Star estimator. Audibert showed that this method is deviation-optimal for finite aggregation  --- the first such result, followed by other estimators with similar properties \citep{lecue2009aggregation,dai2012deviation} for the finite case. We present analysis that quantifies the behavior of this method for arbitrary classes of functions. The method has several nice features. First, it provides an alternative to the 3-stage discretization method of \cite{RakSriTsy15}, does not require the prior knowledge of the entropy of the class, and goes beyond the bounded case.  Second, it enjoys an upper bound of offset Rademacher complexity via relatively routine arguments under rather weak assumptions. Third, it naturally reduces to empirical risk minimization for convex classes (indeed, this happens whenever $\star(\F,\widehat{g})=\F$).

Let $f^*$ denote the minimizer 
$$
f^* = \argmin{f \in \F} \En (f(X) - Y)^2,
$$
and let $\xi$ denote the ``noise''
$$
\xi = Y - f^*.
$$
We say that the model is misspecified if the regression function $\En[Y|X=x]\notin \F$, which means $\xi$ is not zero-mean. Otherwise, we say that the model is well-specified.

\section{A Geometric Inequality}

We start by proving a geometric inequality for the Star estimator. This deterministic inequality holds conditionally on $X_1,\ldots,X_n$, and therefore reduces to a problem in $\reals^n$.

\begin{lemma}[Geometric Inequality]
	\label{lem:angle_ineq}
	The two-step estimator $\pred$ in \eqref{eq:def_estimator} satisfies
\begin{align}
	\label{claim.angle}
	\emp(h - Y)^2 - \emp(\pred - Y)^2 \geq c \cdot \emp(\pred - h)^2 
\end{align}
for any $h\in\F$ and $c= 1/18$. If $\F$ is convex, \eqref{claim.angle} holds with $c=1$. Moreover, if $\F$ is a linear subspace, \eqref{claim.angle} holds with equality and $c=1$ by the Pythagorean theorem. 
\end{lemma}
\begin{remark}
	In the absence of convexity of $\F$, the two-step estimator $\pred$ mimics the key Pythagorean identity, though with a constant $1/18$. We have not focused on optimizing $c$ but rather on presenting a clean geometric argument. 
\end{remark}
\begin{proof}[\textbf{Proof of Lemma~\ref{lem:angle_ineq}}]
	\begin{figure}[h]
		\centering
			\includegraphics[width=.3\textwidth]{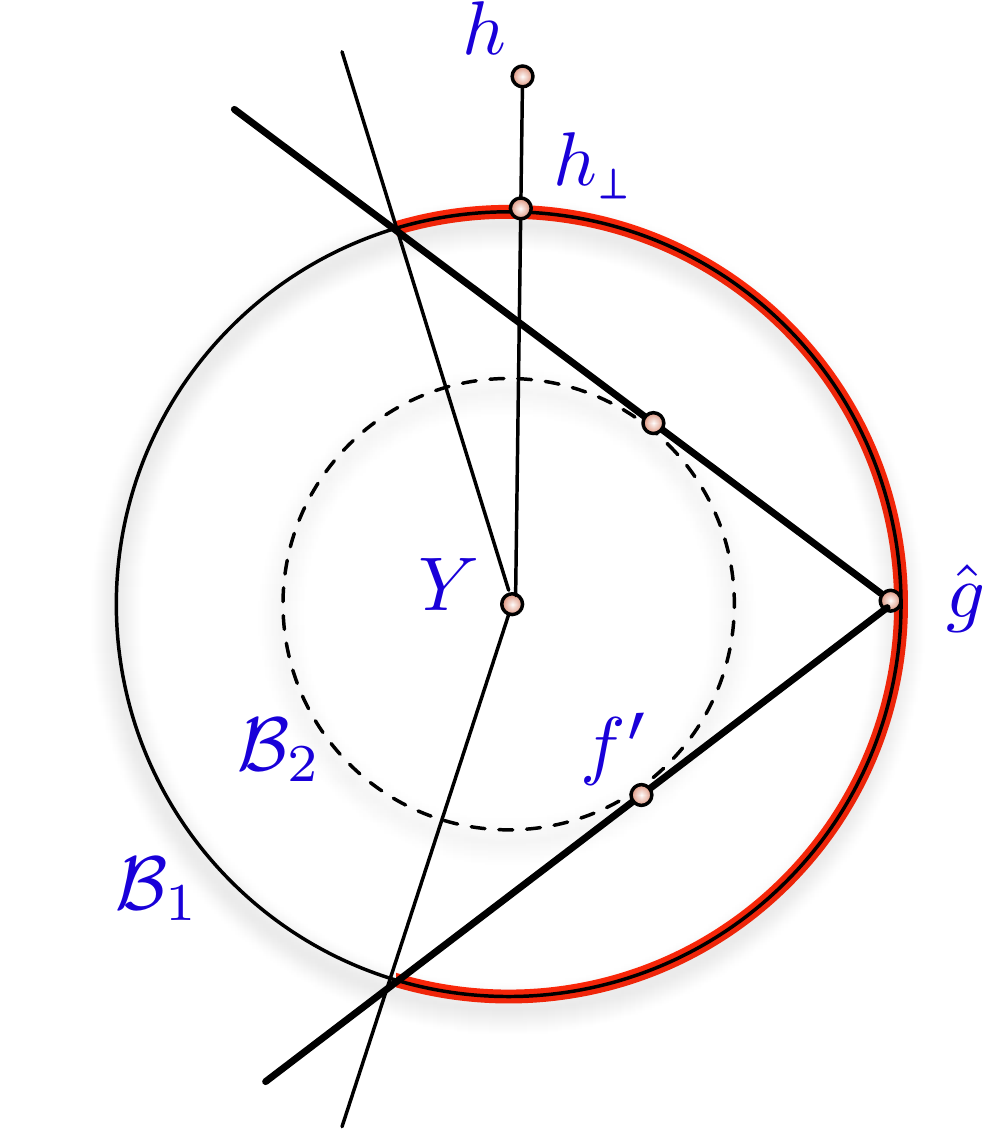}
		\label{fig:geometry}
	\end{figure}

Define the empirical $\ell_2$ distance to be, for any $f,g$, $\| f  \|_{n}:= [\emp f^2]^{1/2}$ and empirical product to be $\langle f, g\rangle_{n} := \emp [f g]$. We will slightly abuse the notation by identifying every function with its finite-dimensional projection on $(X_1,\ldots,X_n)$.

Denote the ball (and sphere) centered at $Y$ and with radius $\|\widehat{g}-Y\|_n$ to be $\B_{1}: = \B(Y, \| \widehat{g} - Y\|_{n})$ (and $\S_1$, correspondingly). In a similar manner, define $\B_2: = \B(Y, \| \pred - Y\|_{n})$ and $\S_2$. By the definition of the Star algorithm, we have $\B_{2} \subseteq \B_{1}$. The statement holds with $c=1$ if $\pred = \widehat{g}$, and so we may assume $\B_2\subset \B_1$. Denote by $\C$ the conic hull of $\B_2$ with origin at $\widehat{g}$. Define the spherical cap outside the cone $\C$ to be $\S = \S_1 \setminus \C$ (drawn in red in Figure~\ref{fig:geometry}).

First, by the optimality of $\widehat{g}$, for any $h \in \F$,  we have $\| h-Y\|_{n}^2 \geq  \|\widehat{g}-Y\|_{n}^2$, i.e. any $h\in \F$ is not in the interior of $\B_1$. Furthermore, $h$ is not in the interior of the cone $\C$, as otherwise there would be a point inside $\B_2$ strictly better than $\pred$. Thus $h \in (\text{int} \C)^c \cap (\text{int} \B_1)^c$.

Second, $\pred \in \B_2$ and it is a contact point of $\C$ and $\S_2$. Indeed, $\pred$ is necessarily on a line segment between $\hat{g}$ and a point outside $\B_1$ that does not pass through the interior of $\B_2$ by optimality of $\pred$. Let $K$ be the set of all contact points -- potential locations of $\pred$.

Now we fix $h\in\F$ and consider the two dimensional plane $\mathcal{L}$ that passes through three points $(\hat{g}, Y, h)$, depicted in Figure~\ref{fig:geometry}. Observe that the left-hand-side of the desired inequality \eqref{claim.angle} is constant as $\pred$ ranges over $K$. To prove the inequality it therefore suffices to choose a value $f'\in K$ that maximizes the right-hand-side. The maximization of $\|h-f'\|^2$ over $f'\in K$ is achieved by $f'\in K\cap \mathcal{L}$. This can be argued simply by symmetry: the two-dimensional plane $\mathcal{L}$ intersects ${\sf span}(K)$ in a line and the distance between $h$ and $K$ is maximized at the extreme point of this intersection. Hence, to prove the desired inequality, we can restrict our attention to the plane $\mathcal{L}$ and $f'$ instead of $\pred$. 

For any $h \in \F$, define the projection of $h$ onto the shell $\mathcal{L}\cap \S$ to be $h_{\perp} \in \S$. We first prove \eqref{claim.angle} for $h_{\perp}$ and then extend the statement to $h$. By the geometry of the cone, 
	$$
	\| f' - \widehat{g} \|_{n}  \geq \frac{1}{2} \| \widehat{g} - h_{\perp}\|_{n} .
	$$
	By triangle inequality,
	\begin{align*}
	\|f'- \widehat{g}\|_{n} \geq \frac{1}{2} \|\widehat{g} - h_{\perp}\|_{n} \geq \frac{1}{2} \left( \|f'- h_{\perp}\|_{n} - \|f' - \widehat{g} \|_{n} \right) .
	\end{align*}
	Rearranging,
	\begin{align*}
	\| f' - \widehat{g}\|_{n}^2 \geq \frac{1}{9} \|f' - h_{\perp}\|_{n}^2 .
	\end{align*}
	By the Pythagorean theorem, 
    $$
	\| h_{\perp} - Y\|_{n}^2 - \|f' - Y\|_{n}^2 = \|\widehat{g} - Y\|_{n}^2 - \|f' - Y\|_{n}^2  = \|f' - \widehat{g}\|_{n}^2 \geq \frac{1}{9} \|f' - h_{\perp}\|_{n}^2,
	$$
	thus proving the claim for $h_{\perp}$ for constant $c = 1/9$. 
	
	We can now extend the claim to $h$. Indeed, due to the fact that $h \in (\text{int} \C)^c \cap (\text{int} \B_1)^c$ and the geometry of the projection $h \rightarrow h_{\perp}$, we have $\langle h_{\perp} - Y, h_{\perp} - h \rangle_{n} \leq 0$. Thus
	\begin{align*}
	\| h - Y\|_{n}^2 - \|f'-Y \|_{n}^2 & = \|  h_{\perp} - h \|_{n}^2 + \| h_{\perp} - Y\|_{n}^2 -2\langle h_{\perp} - Y, h_{\perp} - h \rangle_{n} -\|f'-Y \|_{n}^2  \\
	& \geq  \|  h_{\perp} - h \|_{n}^2 + (\| h_{\perp} - Y\|_{n}^2-\|f'-Y \|_{n}^2) \\
	& \geq \|  h_{\perp} - h \|_{n}^2+ \frac{1}{9} \| f' - h_{\perp} \|_{n}^2 \geq \frac{1}{18} (\|  h_{\perp} - h \|_{n}+\|f' - h_{\perp} \|_{n})^2 \\
	& \geq \frac{1}{18} \| f'-h \|_{n}^2.
	\end{align*}
	This proves the claim for $h$ with constant $1/18$.

\end{proof}

An upper bound on excess loss follows immediately from Lemma~\ref{lem:angle_ineq}.
\begin{corollary}
	\label{cor:excess_loss_bound_deterministic}
Conditioned on the data $\{X_n, Y_n\}$, we have a deterministic upper bound for the Star algorithm:
	\begin{align}
		\label{eq:excess_loss_bound_deterministic}
		\excessloss(\pred) &\leq (\emp-\En)  [2(f^*-Y)(f^* - \pred)] + \En (f^* - \pred)^2 - (1+c) \cdot \emp(f^* - \pred)^2, 
	\end{align}
	with the value of constant $c$ given in Lemma~\ref{lem:angle_ineq}.
\end{corollary}
\begin{proof}
	\begin{align*}
		\excessloss(\pred) & =  \En(\pred(X) - Y)^2 - \inf_{f \in \F} \En(f(X)-Y)^2 \\
		& \leq  \En(\pred - Y)^2 - \En(f^*-Y)^2 + \left[ \emp(f^* - Y)^2 - \emp(\pred - Y)^2 - c \cdot \emp(\pred - f^*)^2 \right] \\
		& = (\emp-\En)  [2(f^*-Y)(f^* - \pred)] + \En (f^* - \pred)^2 - (1+c) \cdot \emp(f^* - \pred)^2. 
	\end{align*}	
\end{proof}
An attentive reader will notice that the multiplier on the negative empirical quadratic term in \eqref{eq:excess_loss_bound_deterministic} is slightly larger than the one on the expected quadratic term. This is the starting point of the analysis that follows.

\section{Symmetrization}

We will now show that the discrepancy in the multiplier constant in \eqref{eq:excess_loss_bound_deterministic} leads to offset Rademacher complexity through rather elementary symmetrization inequalities. We perform this analysis both in expectation (for the case of bounded functions) and in high probability (for the general unbounded case). While the former result follows from the latter, the in-expectation statement for bounded functions requires no assumptions, in contrast to control of the tails. 

\begin{theorem}
	\label{thm:bounded_offset}
	Define the set $\H : = \F - f^* + \star(\F - \F) $. The following expectation bound on excess loss of the Star estimator holds:
	\begin{align*}
		\En\excessloss(\pred) \leq (2M + K(2+c)/2) \cdot \En \sup_{h \in \H}  \left\{ \frac{1}{n}\sum_{i=1}^n 2 \epsilon_i h(X_i) - c' h(X_i)^2 \right\}
	\end{align*}
	where $\epsilon_1,\ldots,\epsilon_n$ are independent Rademacher random variables, $c' = \min\{ \frac{c}{4M},\frac{c}{4K(2+c)}\}$, $K=\sup_{f } |f|_\infty$, and $M = \sup_{f} |Y-f |_\infty$ almost surely.
\end{theorem}

The proof of the theorem involves an introduction of independent Rademacher random variables and  two contraction-style arguments to remove the multipliers $(Y_i-f^*(X_i))$. These algebraic manipulations are postponed to the appendix.

The term in the curly brackets will be called an offset Rademacher process, and the expected supremum --- an offset Rademacher complexity. While Theorem~\ref{thm:bounded_offset} only applies to bounded functions and bounded noise, the upper bound already captures the localization phenomenon, even for non-convex function classes (and thus goes well beyond the classical local Rademacher analysis). 

As argued in \citep{Mendelson14}, it is the contraction step that requires boundedness of the functions when analyzing square loss. Mendelson uses a small ball assumption (a weak condition on the distribution, stated below) to split the analysis into the study of the multiplier and quadratic terms. This assumption allows one to compare the expected square of any function to its empirical version, to within a multiplicative constant that depends on the small ball property. In contrast, we need a somewhat stronger assumption that will allow us to take this constant to be at least $1-c/4$. We phrase this  condition---the lower isometry bound---as follows.
\footnote{We thank Shahar Mendelson for pointing out that the small ball condition in the initial version of this paper was too weak for our purposes.}
\begin{definition}[Lower Isometry Bound]
\label{Assump:Low-Iso-Bd}
	We say that a function class $\F$ satisfies the lower isometry bound with some parameters $0<\eta<1$ and $0<\delta<1$ if
	\begin{align}
		\mbb{P}\left( \inf_{f \in \F \setminus \{ 0\}}\frac{1}{n} \sum_{i=1}^n \frac{f^2(X_i)}{\En f^2}  \geq 1 - \eta   \right) \geq 1 - \delta
	\end{align}
	for all $n \geq n_0(\F, \delta, \eta)$, where $n_0(\F, \delta, \eta)$ depends on the complexity of the class.
\end{definition}
In general this is a mild assumption that requires good tail behavior of functions in $\F$, yet it is stronger than the small ball property. Mendelson \cite{Mendelson15} shows that this condition holds for heavy-tailed classes assuming the small ball condition plus a norm-comparison  property $\| f \|_{\ell_q} \leq L \| f \|_{\ell_2}, \forall f \in \F $. We also remark that Assumption~\ref{Assump:Low-Iso-Bd} holds for sub-gaussian classes $\F$ using concentration tools, as already shown in \cite{lecue2013learning}. For completeness, let us also state the small ball property:
\begin{definition}[Small Ball Property \cite{Mendelson14,Mendelson14general}]
	The class of functions $\F$ satisfies the small-ball condition if there exist constants $\kappa>0$ and $0<\epsilon<1$ for every $f\in \F$,
	$$\mathbb{P}\big(|f(X)| \geq \kappa (\En f^2)^{1/2} \big) \geq \epsilon.$$
\end{definition}

Armed with the lower isometry bound, we now prove that the tail behavior of the deterministic upper bound in  \eqref{eq:excess_loss_bound_deterministic} can be controlled via the tail behavior of  offset Rademacher complexity.
\begin{theorem}
	\label{thm:unbounded-offset}
	Define the set $\H : = \F - f^* + \star(\F - \F) $. Assume the lower isometry bound in Definition~\ref{Assump:Low-Iso-Bd} holds with 
		$\eta = c/4$ and some $\delta<1$,
	 where $c$ is the constant in \eqref{claim.angle}. 
	Let $\xi_i = Y_i - f^*(X_i)$. Define $$
A := \sup_{h \in \H} \frac{\En h^4}{(\En h^2)^2} ~~~\text{and}~~~ B:= \sup_{X, Y} \En \xi^4.
$$
Then there exist two absolute constants $c', \tilde{c}>0$ (only depends on $c$), such that 
	 	\begin{align*}
		&\mbb{P}\left( \excessloss(\pred) > 4 u \right) \leq  4 \delta + 4 \mbb{P}\left( \sup_{h \in \H} \frac{1}{n} \sum_{i=1}^n \epsilon_i \xi_i h(X_i) - \tilde{c} \cdot  h(X_i)^2  > u  \right)
		\end{align*}
	for any $$u>\frac{32\sqrt{AB}}{c'} \cdot \frac{1}{n},$$ as long as
	$n > \frac{16 (1-c')^2 A}{c'^2} \vee n_0(\H,\delta,c/4)$.
\end{theorem}

Theorem~\ref{thm:unbounded-offset} states that excess loss is stochastically dominated by offset Rademacher complexity.  We remark that the requirement in $A,B$ holds under the mild moment conditions.

\begin{remark}
In certain cases, Definition~\ref{Assump:Low-Iso-Bd} can be shown to hold for $f\in \F \setminus r^* \B$ (rather than all $f\in\F$), for some critical radius $r^*$, as soon as $n \geq n_0(\F,\delta,\eta,r^*)$  (see \cite{Mendelson15}). In this case, the bound on the offset complexity is only affected additively by $(r^*)^2$.
\end{remark}

We postpone the proof of the Theorem to the appendix. In a nutshell, it extends the classical probabilistic symmetrization technique \citep{GinZin84,Men03fewnslt} to the non-zero-mean offset process under the investigation.

\section{Offset Rademacher Process: Chaining and Critical Radius}

Let us summarize the development so far. We have shown that excess loss of the Star estimator is upper bounded by the (data-dependent) offset Rademacher complexity, both in expectation and in high probability, under the appropriate assumptions. We claim that the necessary properties of the estimator are now captured by the offset complexity, and we are now squarely in the realm of empirical process theory. In particular, we may want to quantify rates of convergence under complexity assumptions on $\F$, such as covering numbers. In contrast to local Rademacher analyses where one would need to estimate the data-dependent fixed point of the critical radius in some way, the task is much easier for the offset complexity. To this end, we study the offset process with the tools of empirical process theory. 

\subsection{Chaining Bounds}
The first lemma describes the behavior of offset Rademacher process for a finite class. 
\begin{lemma}
	\label{eq:finite}
	Let $V\subset \reals^n$ be a finite set of vectors of cardinality $N$. Then for any $C>0$,
	\begin{align*}
		\En_\epsilon \max_{v\in V} \left[ \frac{1}{n}\sum_{i=1}^n \epsilon_i v_i - Cv_i^2 \right] \leq \frac{1}{2C} \frac{\log N}{n}.
	\end{align*}
	Furthermore, for any $\delta>0$, 
	\begin{align*}
		\Pr{\max_{v\in V} \left[ \frac{1}{n}\sum_{i=1}^n \epsilon_i v_i - Cv_i^2 \right] \geq \frac{1}{2C} \frac{\log  N+ \log 1/\delta}{n} } \leq \delta.
	\end{align*}
When the noise $\xi $ is unbounded, 
\begin{align*}
		&\En_\epsilon \max_{v \in V }  \left[ \frac{1}{n}\sum_{i=1}^n \epsilon_i \xi_i v_i - C v_i^2 \right]  \leq  M \cdot \frac{\log N }{n}, \\
		&\mbb{P}_{\epsilon} \left( \max_{v\in V} \left[ \frac{1}{n}\sum_{i=1}^n \epsilon_i \xi_i v_i - Cv_i^2 \right] \geq  M  \cdot \frac{\log  N+ \log 1/\delta}{n}  \right) \leq \delta,
	\end{align*}
	where 
	\begin{align}
	\label{eq:M}
	M : = \sup_{v \in V \setminus \{0\}} \frac{ \sum_{i=1}^n v_i^2 \xi_i^2}{2C \sum_{i=1}^n v_i^2}.
	\end{align}
\end{lemma}

Armed with the lemma for a finite collection, we upper bound the offset Rademacher complexity of a general class through the chaining technique. We perform the analysis in expectation and in probability. Recall that a $\delta$-cover of a subset $S$ in a metric space $(T,d)$ is a collection of elements such that the union of the $\delta$-balls with centers at the elements contains $S$. A covering number at scale $\delta$ is the size of the minimal $\delta$-cover. 

	One of the main objectives of symmetrization is to arrive at a stochastic process that can be studied conditionally on data, so that all the relevant complexities can be made sample-based (or, empirical). Since the functions only enter offset Rademacher complexity through their values on the sample $X_1,\ldots,X_n$, we are left with a finite-dimensional object. Throughout the paper, we work with the empirical $\ell_2$ distance	
	$$d_n(f,g) = \left( \frac{1}{n}\sum_{i=1}^n (f(X_i)-g(X_i))^2\right)^{1/2}.$$
	The covering number of $\G$ at scale $\delta$ with respect to $d_n$ will be denoted by $\cN_2(\G,\delta)$. 
	
\begin{lemma}
	\label{lem:offset_estimate_chaining}
	Let $\G$ be a class of functions from $\Z$ to ~$\reals$. Then for any $z_1,\ldots,z_n \in \Z$
	\begin{align*}
		\En_\epsilon \sup_{g\in\G} \left[ \frac{1}{n}\sum_{t=1}^n \epsilon_i g(z_i) - Cg(z_i)^2 \right] &\leq \inf_{\gamma\geq 0, \alpha\in [0,\gamma]} \left\{
			\frac{(2/C) \log \cN_2(\G, \gamma)}{n} \right.\\
			&\left.\hspace{1.5in} + 4\alpha + \frac{12}{\sqrt{n}}\int_{\alpha}^\gamma \sqrt{\log\cN_2(\G, \delta)}d\delta
		\right\}
	\end{align*}
	where $\cN_2(\G,\gamma)$ is an $\ell_2$-cover of $\G$ on $(z_1,\ldots,z_n)$ at scale $\gamma$ (assumed to contain $\bf{0}$).
\end{lemma}

Instead of assuming that $\bf{0}$ is contained in the cover, we may simply increase the size of the cover by $1$, which can be absorbed by a small change of a constant.

Let us discuss the upper bound of Lemma~\ref{lem:offset_estimate_chaining}. First, we may take $\alpha=0$, unless the integral diverges (which happens for very large classes with entropy growth of $\log \cN_2(\G,\delta) \sim \delta^{-p}$, $p\geq 2$). Next, observe that first term is precisely the rate of aggregation with a finite collection of size $\cN_2(\G,\gamma)$. Hence, the upper bound is an optimal balance of the following procedure: cover the set at scale $\gamma$ and pay the rate of aggregation for this finite collection, plus pay the rate of convergence of ERM within a $\gamma$-ball. The optimal balance is given by some $\gamma$ (and can be easily computed under assumptions on covering number behavior --- see \citep{RakSri14nonparam}). The optimal $\gamma$ quantifies the localization radius that arises from the curvature of the loss function. One may also view the optimal balance as the well-known equation
$$\frac{\log \cN(\G,\gamma)}{n} \asymp \gamma^2,$$
studied in statistics \citep{yang1999information} for well-specified models. The present paper, as well as \citep{RakSriTsy15}, extend the analysis of this balance to the misspecified case and non-convex classes of functions.

Now we provide a high probability analogue of Lemma \ref{lem:offset_estimate_chaining}.
\begin{lemma}
\label{lma:prob-chaining}
Let $\G$ be a class of functions from $\Z$ to ~$\reals$. Then for any $z_1,\ldots,z_n \in \Z$ and any $u>0$,
\small{
	\begin{align*}
		& \quad \mbb{P}_\epsilon \left( \sup_{g\in\G} \left[ \frac{1}{n}\sum_{t=1}^n \epsilon_i g(z_i) - Cg(z_i)^2 \right] > u \cdot \inf_{\alpha \in [0,\gamma]} \left\{ 4\alpha + \frac{12}{\sqrt{n}} \int_{\alpha}^{\gamma} \sqrt{\log \mc{N}_2(\G,\delta)} d\delta \right\}  +  \frac{2}{C} \frac{\log \mc{N}_2(\G,\gamma) + u}{n}  \right)\\
		& \leq \frac{2}{1-e^{-2}} \exp(-cu^2) + \exp(-u)
	\end{align*}}
	where $\cN_2(\G,\gamma)$ is an $\ell_2$-cover of $\G$ on $(z_1,\ldots,z_n)$ at scale $\gamma$ (assumed to contain $\bf{0}$) and $C,c>0$ are universal constants.
\end{lemma}

The above lemmas study the behavior of offset Rademacher complexity for abstract classes $\G$. Observe that the upper bounds in previous sections are in terms of the class $\F-f^*+\star(\F-\F)$. This class, however, is not more complex that the original class $\F$ (with the exception of a finite class $\F$). More precisely, the covering numbers of $\F + \F' : = \{ f+g: f \in \F, g \in \F' \}$ and $\F - \F' := \{ f -g:  f \in \F , g \in \F' \}$ are bounded as
\begin{align*}
\log \mc{N}_{2}(\F+\F', 2 \epsilon), ~ \log \mc{N}_{2}(\F-\F', 2 \epsilon) \leq \log \mc{N}_2(\F, \epsilon) + \log \mc{N}_2(\F', \epsilon)
\end{align*}
for any $\F,\F'$. The following lemma shows that the complexity of the star hull $\star(\F)$ is also  not significantly larger than that of $\F$. 
\begin{lemma}[\cite{mendelson2002improving}, Lemma 4.5]
\label{lma:star-cov}
	For any scale $\epsilon>0$, the covering number of $\F \subset \B_2$ and that of  $\star(\F)$ are bounded in the sense
	\begin{align*}
		\log \mc{N}_2(\F,2\epsilon) \leq \log \mc{N}_2(\star(\F),2\epsilon) \leq \log \frac{2}{\epsilon} + \log \mc{N}_2(\F,\epsilon).
	\end{align*}
\end{lemma}

\subsection{Critical Radius}
Now let us study the critical radius of offset Rademacher processes. Let $\xi = f^*-Y$ and define
\begin{align}
\label{eq:alpha}
\alpha_n(\H, \kappa, \delta ) \deq \inf \left\{ r > 0 : \mbb{P} \left( \sup_{h \in \H \cap r\B} \left\{ \frac{1}{n} \sum_{i=1}^n 2\epsilon_i \xi_i h(X_i) - c' \frac{1}{n} \sum_{i=1}^n h^2(X_i) \right\}  \leq \kappa  r^2 \right) \geq 1-\delta \right\}.
\end{align}
\begin{theorem}
	\label{thm:crit_radius}
	Assume $\H$ is star-shaped around 0 and the lower isometry bound holds for $\delta,\epsilon$. Define the critical radius
$$
r =  \alpha_n(\H, c' (1-\epsilon), \delta) .
$$
Then we have with probability at least $1 -2 \delta $, 
\begin{align*}
	\sup_{h \in \H} \left\{ \frac{2}{n} \sum_{i=1}^n \epsilon_i \xi_i h(X_i) - c' \frac{1}{n} \sum_{i=1}^n h^2(X_i) \right\}  =\sup_{h \in \H \cap r \B} \left\{ \frac{2}{n} \sum_{i=1}^n \epsilon_i \xi_i h(X_i) - c' \frac{1}{n} \sum_{i=1}^n h^2(X_i) \right\},
\end{align*}
which further implies
\begin{align*}
 	\sup_{h \in \H} \left\{ \frac{2}{n} \sum_{i=1}^n \epsilon_i \xi_i h(X_i) - c' \frac{1}{n} \sum_{i=1}^n h^2(X_i) \right\}   \leq  r^2.
\end{align*}
\end{theorem}

The first statement of Theorem~\ref{thm:crit_radius} shows the self-modulating behavior of the offset process: there is a critical radius, beyond which the fluctuations of the offset process are controlled by those within the radius. To understand the second statement, we observe that the complexity $\alpha_n$ is upper bounded by the corresponding complexity in \citep{Mendelson14}, which is defined without the quadratic term subtracted off. Hence, offset Rademacher complexity is no larger (under our Assumption~\ref{Assump:Low-Iso-Bd}) than the upper bounds obtained by \cite{Mendelson14} in terms of the critical radius.

\section{Examples}
\label{sec:examples}
In this section, we briefly describe several applications. The first is concerned with parametric regression.
\begin{lemma}
	Consider the parametric regression $Y_i = X_i^T \beta^* + \xi_i, 1\leq i\leq n$, where $\xi_i$ need not be centered. The offset Rademacher complexity is bounded as
	\begin{align*}
		&\En_{\epsilon}  \sup_{\beta \in \mbb{R}^p } \left\{ \frac{1}{n}\sum_{i=1}^n 2 \epsilon_i \xi_i X_i^T \beta - C \beta^T X_i X_i^T \beta   \right\}  =  \frac{{\sf tr}\left(G^{-1}H\right)}{Cn}
	\end{align*}
	and
	\small{
	\begin{align*}
		&\mbb{P}_{\epsilon} \left(   \sup_{\beta \in \mbb{R}^p } \left\{ \frac{1}{n} \sum_{i=1}^n 2 \epsilon_i \xi_i X_i^T \beta - C \beta^T X_i X_i^T \beta \right\} \geq  \frac{{\sf tr}\left(G^{-1} H\right)}{Cn} + \frac{\sqrt{{\sf tr}\left([G^{-1}H]^2\right) }}{n} (4\sqrt{2\log \frac{1}{\delta}} + 64 \log \frac{1}{\delta}) \right) \leq \delta
	\end{align*}}
	where $G: = \sum_{i=1}^n X_i X_i^T$ is the Gram matrix and $H = \sum_{i=1}^n \xi_i^2 X_i X_i^T$. In the well-specified case (that is, $\xi_i$ are zero-mean), assuming that conditional variance is  $\sigma^2$, then conditionally on the design matrix, $\En G^{-1}H = \sigma^2 I_p$ and excess loss is upper bounded by order $\frac{\sigma^2 p}{n} $.
\end{lemma}
\begin{proof}
The offset Rademacher can be interpreted as the Fenchel-Legendre transform, where
\begin{align}
\label{eq:rad-chaos}
    	\sup_{\beta \in \mbb{R}^p } \left\{ \sum_{i=1}^n 2 \epsilon_i \xi_i X_i^T \beta - C \beta^T X_i X_i^T \beta   \right\} =  \frac{\sum_{i,j=1}^n \epsilon_i \epsilon_j \xi_i \xi_j  X_i^T G^{-1} X_j}{C n}.
\end{align}
Thus we have in expectation
\begin{align}
	\En_{\epsilon} \frac{1}{n} \sup_{\beta \in \mbb{R}^p } \left\{ \sum_{i=1}^n 2 \epsilon_i \xi_i X_i^T \beta - C \beta^T X_i X_i^T \beta   \right\}  = \frac{\sum_{i=1}^n \xi_i^2 X_i^T G^{-1} X_i}{C n} = \frac{{\sf tr}[G^{-1} (\sum_{i=1}^n \xi_i^2 X_i X_i^T )]}{Cn}.
\end{align}
For high probability bound, note the expression in Equation \eqref{eq:rad-chaos} is Rademacher chaos of order two. Define symmetric matrix $M \in \mbb{R}^{n\times n}$ with entries
$$
M_{ij} =  \xi_i \xi_j  X_i^T G^{-1} X_j
$$
and define 
$$
Z = \sum_{i,j=1}^n \epsilon_i \epsilon_j \xi_i \xi_j  X_i^T G^{-1} X_j = \sum_{i,j=1}^n \epsilon_i \epsilon_j M_{ij}.
$$
Then
$$
\En Z = {\sf tr}[G^{-1} (\sum_{i=1}^n \xi_i^2 X_i X_i^T )],
$$
and
$$
\En \sum_{i=1}^n(\sum_{j=1}^n \epsilon_j M_{ij})^2 = \| M \|_{F}^2 = {\sf tr}[G^{-1}(\sum_{i=1}^n \xi_i^2 X_i X_i^T)G^{-1}(\sum_{i=1}^n \xi_i^2 X_i X_i^T )].
$$
Furthermore,
$$
\| M \| \leq \| M \|_{F} = \sqrt{{\sf tr}[G^{-1}(\sum_{i=1}^n \xi_i^2 X_i X_i^T)G^{-1}(\sum_{i=1}^n \xi_i^2 X_i X_i^T )] }
$$
We apply the concentration result in \citep[Exercise 6.9]{boucheron2013concentration},
\begin{align}
	\mbb{P} \left( Z - \mbb{E} Z \geq 4\sqrt{2} \| M \|_F  \sqrt{t} + 64 \| M \| t  \right) \leq e^{-t}.
\end{align}
\end{proof}

For the finite dictionary aggregation problem, the following lemma shows control of offset Rademacher complexity.  
\begin{lemma}
	\label{lem:finite_agg}
Assume $\F \in \B_2$ is a finite class of cardinality $N$.  Define $\H = \F -f^* + \star(\F - \F)$ which contains the Star estimator $\pred -f^*$ defined in Equation \eqref{eq:def_estimator}. The offset Rademacher complexity for $\H$ is bounded as
\begin{align*}
		&\En_{\epsilon}  \sup_{h \in \H} \left\{ \frac{1}{n}\sum_{i=1}^n 2 \epsilon_i \xi_i h(X_i) - C h(X_i)^2 \right\}  
		\leq \tilde{C} \cdot \frac{\log (N \vee n)}{n}
\end{align*}
and
\begin{align*}
\mbb{P}_{\epsilon} \left(  \sup_{h \in \H} \left\{ \frac{1}{n}\sum_{i=1}^n 2 \epsilon_i \xi_i h(X_i) - C h(X_i)^2 \right\}  \leq   \tilde{C} \cdot \frac{\log (N\vee n) + \log \frac{1}{\delta} }{n} \right) \leq \delta .
\end{align*}
where $\tilde{C}$ is a constant depends on $K:=2  (  \sqrt{\sum_{i=1}^n \xi_i^2/n} + 2C )$ and 
$$M : = \sup_{h \in \H \setminus \{0\}} \frac{ \sum_{i=1}^n h(X_i)^2 \xi_i^2}{2C \sum_{i=1}^n h(X_i)^2}.$$
\end{lemma}

	We observe that the bound of Lemma~\ref{lem:finite_agg} is worse than the optimal bound of \citep{audibert2007progressive} by an additive $\frac{\log n}{n}$ term. This is due to the fact that the analysis for finite case passes through the offset Rademacher complexity of the star hull, and for this case the star hull is more rich than the finite class. For this case, a direct analysis of the Star estimator is provided in \citep{audibert2007progressive}.

While the offset complexity of the star hull is crude for the finite case, the offset Rademacher complexity \emph{does} capture the correct rates for regression with larger classes, initially derived in \citep{RakSriTsy15}. We briefly mention the result. The proof is identical to the one in \citep{RakSri14nonparam}, with the only difference that offset Rademacher is defined in that paper as a sequential complexity in the context of online learning.

\begin{corollary}
	Consider the problem of nonparametric regression, as quantified by the growth 
	$$\log \cN_2(\F,\epsilon) \leq \epsilon^{-p}.$$ 
	In the regime $p\in(0,2)$, the upper bound of Lemma~\ref{lma:prob-chaining} scales as $n^{-\frac{2}{2+p}}$. In the regime $p\geq 2$, the bound scales as $n^{-1/p}$, with an extra logarithmic factor at $p=2$. 
\end{corollary}

For the parametric case of $p=0$, one may also readily estimate the offset complexity. Results for VC classes, sparse combinations of dictionary elements, and other parametric cases follow easily by plugging in the estimate for the covering number or directly upper bounding the offset complexity (see \cite{RakSriTsy15, RakSri14nonparam}).

\section{Lower bound on Minimax Regret via Offset Rademacher Complexity}

We conclude this paper with a lower bound on minimax regret in terms of offset Rademacher complexity. \begin{theorem}[Minimax Lower Bound on Regret]
\label{thm: mini-low-bd}
         Define the offset Rademacher complexity over $\X^{\otimes n}$ as
         \begin{align*}
         \Rad^{\sf  o}(n, \F) = \sup_{\{x_i\}_{i=1}^n \in \X^{\otimes n}} \En_{\epsilon} \sup_{f \in \F} \left\{ \frac{1}{n} \sum_{i=1}^n 2 \epsilon_i f(x_i) -f(x_i)^2 \right\}
         \end{align*}
	then the following minimax lower bound on regret holds:
	\begin{align*}
	\inf_{\hat{g} \in \G} \sup_{P}  \left\{ \En (\hat{g} - Y)^2 - \inf_{f \in \F} \En (f - Y)^2 \right\} \geq \Rad^{\sf  o}((1+c)n, \F) - \frac{c}{1+c}\Rad^{\sf o}(cn, \G),
	\end{align*}
	for any $c>0$.
\end{theorem}
For the purposes of matching the performance of the Star procedure, we can take $\G=\F+\star(\F-\F)$. 

\appendix

\newpage
\section{Proofs}

\begin{proof}[\textbf{Proof of Theorem~\ref{thm:bounded_offset}}]
Since $\pred$ is in the star hull around $\widehat{g}$, $\pred$ must lie in the set $\H : = \F + \star(\F-\F)$.
Hence, in view of \eqref{eq:excess_loss_bound_deterministic}, excess loss $\excessloss(\pred)$ is upper bounded by
	\begin{align}		
		 &\sup_{f \in \H} \left\{ (\emp-\En)  [2(f^*-Y)(f^* - f)] + \En (f^* - f)^2 - (1+c) \cdot \emp(f^* - f)^2  \right\} \label{emp.quad.low.2.bounded} \\
		& \leq \sup_{f \in \H} \left\{ (\emp-\En)  [2(f^*-Y)(f^* - f)] + (1+c/4)\En (f^* - f)^2 - (1+3c/4) \cdot \emp(f^* - f)^2 \right. \notag\\
		&\left.\hspace{3in}- (c/4) \left( \emp(f^* - f)^2  + \En(f^*-f)^2 \right) \right\} \notag\\
		&\leq \sup_{f \in \H} \left\{ (\emp-\En)  [2(f^*-Y)(f^* - f)]  - (c/4) \left( \emp(f^* - f)^2  + \En(f^*-f)^2 \right) \right\} \label{eq:frt}\\
	        &+ \sup_{f \in \H} \left\{ (1+c/4)\En (f^* - f)^2 - (1+3c/4) \cdot \emp(f^* - f)^2  \right\}  \label{eq:sec}
	\end{align}
We invoke the supporting Lemma~\ref{lem:contraction} (stated and proved below) for the term \eqref{eq:sec}:
\begin{align}
	&\En \sup_{f \in \H} \left\{ (1+c/4)\En (f^* - f)^2 - (1+3c/4) \cdot \emp(f^* - f)^2  \right\}  \\
	&\leq \frac{K(2+c)}{2} \cdot  \En \sup_{f \in \H} \frac{1}{n} \left\{ \sum_{i=1}^n 2 \epsilon_i (f(X_i)-f^*(X_i)) - \frac{c}{4K(2+c)}\cdot \sum_{i=1}^n  (f(X_i)-f^*(X_i))^2 \right\}.
\end{align}
Let $\emp'$ stand for empirical expectation with respect to an independent copy $(X'_1,\ldots,X'_n)$. For the term \eqref{eq:frt}, Jensen's inequality yields
\begin{align*}
	&\En \sup_{f \in \H} \left\{ (\emp-\En)  [2(f^*-Y)(f^* - f)]  - (c/4) \left( \emp(f^* - f)^2  + \En(f^*-f)^2 \right) \right\} \\
	& \leq \En \sup_{f \in \H} \left\{ (\emp-\emp')  [2(f^*-Y)(f^* - f)]  - (c/4) \left( \emp(f^* - f)^2  + \emp'(f^*-f)^2 \right)\right\}. 
\end{align*}
When introducing i.i.d. Rademacher random variables, we observe that the quadratic term remains unchanged by renaming $X_i$ and $X_i'$, and thus the preceding expression is upper bounded by
\begin{align*}
	2\En \sup_{f \in \H} \left\{ \frac{1}{n}\sum_{i=1}^n 2\epsilon_i(f^*(X_i)-Y_i)(f^*(X_i) - f(X_i))  - (c/4)(f^*(X_i) - f(X_i))^2 \right\}. 
\end{align*}
Using a contraction technique as in the proof of Lemma~\ref{lem:contraction}, we obtain an upper bound of
\begin{align}
	&2M\cdot  \En  \sup_{f \in \H} \frac{1}{n}\left\{ \sum_{i=1}^n 2\epsilon_i (f^*(X_i) - f(X_i))  - \frac{c}{4M} \cdot \sum_{i=1}^n (f^*(X_i) - f(X_i))^2 \right\}
\end{align}
Combining the bounds yields the statement of the theorem.
\end{proof}

\begin{lemma}
\label{lem:contraction}
	For any class $\F$ of uniformly bounded functions with $K=\sup_{f\in\F} |f|_\infty$, for any $f^*\in\F$, and for any $c>0$, it holds that
	\begin{align*}
		&\En\sup_{f\in\F}\left\{ \En(f-f^*)^2 - (1+2c)\emp(f-f^*)^2 \right\}\\
		&\leq c \cdot \En \sup_{f\in\F} \frac{1}{n} \left\{ \frac{4K(1+c)}{c}\sum_{i=1}^n \epsilon_i (f(X_i)-f^*(X_i)) - \sum_{i=1}^n  (f(X_i)-f^*(X_i))^2 \right\}.
	\end{align*}
\end{lemma}
\begin{proof}[\textbf{Proof of Lemma~\ref{lem:contraction}}]
We write
	\begin{align*}
		&\En\sup_{f\in\F}\left\{ \En(f-f^*)^2 - (1+2c)\emp(f-f^*)^2 \right\}\\
		&=\En\sup_{f\in\F}\left\{ (1+c)\En(f-f^*)^2 - (1+c)\emp(f-f^*)^2  - c\En(f-f^*)^2-c\emp(f-f^*)^2 \right\}
	\end{align*}
	which, by Jensen's inequality, is upper bounded by
	\begin{align*}
		&\En\sup_{f\in\F} \left\{ (1+c)(\emp'(f-f^*)^2-\emp(f-f^*)^2) - c\emp'(f-f^*)^2 - c\emp(f-f^*)^2 \right\} 
	\end{align*}
	We recall that $\emp'$ is an empirical mean operator with respect to an independent copy $(X_1',\ldots,X_n')$. Writing out the empirical expectations in the above expression, the above is equal to
	\begin{align*}
		&\En\sup_{f\in\F} \left\{ \frac{1+c}{n}\sum_{i=1}^n \epsilon_i \Big((f(X_i')-f^*(X_i'))^2-(f(X_i)-f^*(X_i))^2 \Big) - c\emp'(f-f^*)^2 - c\emp(f-f^*)^2 \right\}\\
		&\leq 2 \cdot \En\sup_{f\in\F} \left\{ \frac{1+c}{n}\sum_{i=1}^n \epsilon_i (f(X_i)-f^*(X_i))^2 - c\emp(f-f^*)^2 \right\}
	\end{align*}
with the last expectation taken over $\epsilon_i$ and data $X_i$, $1\leq i\leq n$.
	
	We proceed with a contraction-style proof. Condition on $X_1,\ldots,X_n$ and $\epsilon_2,\ldots,\epsilon_n$, and write out the expectation with respect to $\epsilon_1$:
	\begin{align*}
		& \quad  \frac{1}{2} \sup_{f\in\F} \left\{ \frac{1+c}{n}\sum_{i=2}^n \epsilon_i(f(X_i)-f^*(X_i))^2 - c\emp(f-f^*)^2  + \frac{1+c}{n}(f(X_1)-f^*(X_1))^2 \right\}	\\
		&+ \frac{1}{2} \sup_{g\in\F} \left\{ \frac{1+c}{n}\sum_{i=2}^n \epsilon_i(g(X_i)-f^*(X_i))^2 - c\emp(g-f^*)^2  - \frac{1+c}{n}(g(X_1)-f^*(X_1))^2  \right\}\\
		&\leq \frac{1}{2} \sup_{f,g\in\F} \left\{ \frac{1+c}{n}\sum_{i=2}^n \epsilon_t(f(X_i)-f^*(X_i))^2 - c\emp(f-f^*)^2  + \frac{1+c}{n}\sum_{i=2}^n \epsilon_t(g(X_i)-f^*(X_i))^2  \right.\\
		&\left.\hspace{3in} - c\emp(g-f^*)^2 + \frac{4K(1+c)}{n}|f(X_1)-g(X_1)| \right\}	
	\end{align*}
	The absolute value can be dropped since the expression is symmetric in $f,g$. We obtain an upper bound of 
	\begin{align*}
		&\frac{1}{2} \sup_{f,g\in\F} \left\{ \frac{1+c}{n}\sum_{i=2}^n \epsilon_t(f(X_i)-f^*(X_i))^2 - c\emp(f-f^*)^2  + \frac{1+c}{n}\sum_{i=2}^n \epsilon_t(g(X_i)-f^*(X_i))^2  \right.\\
		&\left.\hspace{3in} - c\emp(g-f^*)^2 + \frac{4K(1+c)}{n}(f(X_1)-g(X_1)) \right\}	\\
		&=\En_{\epsilon_1}\sup_{f\in\F} \left\{ \frac{1+c}{n}\sum_{i=2}^n \epsilon_i(f(X_i)-f^*(X_i))^2 - c\emp(f-f^*)^2  + \frac{4K(1+c)}{n}\epsilon_1 f(X_1)\right\} 
	\end{align*}
	Proceeding in this fashion for $\epsilon_2$ until $\epsilon_n$, we conclude
	\begin{align*}
		&\En\sup_{f\in\F}\left\{ \En(f-f^*)^2 - (1+2c)\emp(f-f^*)^2 \right\} \\
		&\leq  \En \sup_{f\in\F} \left\{ \frac{4K(1+c)}{n}\sum_{i=1}^n \epsilon_t (f(X_i)-f^*(X_i)) - \frac{c}{n}\sum_{i=1}^n  (f(X_i)-f^*(X_i))^2 \right\}
	\end{align*}
	where we added $f^*$ back in for free since random signs are zero-mean.	
\end{proof}

\begin{proof}[\textbf{Proof of Theorem~\ref{thm:unbounded-offset}}]
	We start with the deterministic upper bound \eqref{emp.quad.low.2.bounded} on excess loss (see the proof of Theorem~\ref{thm:bounded_offset}):
	\begin{align}
		\label{emp.quad.low.2}
		\sup_{h \in \H} \left\{ (\emp-\En)  [2\xi h] + \En h^2 - (1+c) \cdot \emp h^2  \right\}
	\end{align}
	where $h = f- f^* \in \H $.
Define
\begin{align*}
	U_{X_i,Y_i}(h) &= 2\xi_ih(X_i) - \mathbb{E} [2\xi h] + \En h^2 - (1+c) \cdot h(X_i)^2, \\
	V_{X_i,Y_i}(h) &=  2\xi_ih(X_i) - \mathbb{E} [2\xi h] - \En h^2 + (1-c') \cdot h(X_i)^2 .
\end{align*}
where $c'$ will be specified later. We now prove a version of probabilistic symmetrization lemma \cite{GinZin84,Men03fewnslt} for
\begin{align}
	\mbb{P} \left( \sup_{h \in \H}  \sum_{i=1}^n U_{X_i,Y_i}(h)  > x \right).
\end{align}
Note that unlike the usual applications of the technique in the literature, we perform symmetrization with the quadratic terms. Define 
\begin{align}
	\B =\left\{ \sup_{h \in \H} \sum_{i=1}^n U_{X_i,Y_i}(h) > x \right\}, ~~
	\beta = \inf_{h \in \H} \mbb{P} \left( \sum_{i=1}^n V_{X_i,Y_i}(h) < \frac{x}{2} \right).
\end{align}
Clearly for $\{X_i, Y_i\}_{i=1}^n \in \B$, there exists a $h \in \H$ satisfies condition in $\B$. If in addition $h$ satisfies
$$
\sum_{i=1}^n V_{X_i',Y_i'}(h)<\frac{x}{2}
$$
then
$$\sum_{i=1}^n U_{X_i,Y_i}(h) - V_{X_i',Y_i'}(h) > \frac{x}{2}$$
and therefore
$$\sup_{h \in \H} \sum_{i=1}^n U_{X_i,Y_i}(h) - V_{X_i',Y_i'}(h) > \frac{x}{2}.$$
The latter can be written as
\begin{align*}
	&\sup_{h \in \H} \left\{ \sum_{i=1}^n 2\xi_ih(X_i) - 2\xi_i' h(X_i') + 2 \En h^2 - (1+c) \cdot h(X_i)^2 - (1-c') \cdot h(X_i')^2\right\}  > \frac{x}{2} .
\end{align*}
Then for this particular $h$,
\begin{align*}
	\beta & = \inf_{g \in \cH} \mbb{P} \left( \sum_{i=1}^n  V_{X_i',Y_i'}(g) < \frac{x}{2} \right) 
	\leq \mbb{P} \left( \sum_{i=1}^n  V_{X_i',Y_i'}(h) < \frac{x}{2} \right) \\
	& \leq \mbb{P} \left( \sum_{i=1}^n U_{X_i,Y_i}(h) -  V_{X_i',Y_i'}(h) > \frac{x}{2} \right) 
	\leq \mbb{P} \left(\sup_{h \in \H}  \sum_{i=1}^n U_{X_i,Y_i}(h) -  V_{X_i',Y_i'}(h) > \frac{x}{2} \right).
\end{align*}
Note that the right-hand-side does not depend on $h$. We integrate over $\{X_i, Y_i\}_{i=1}^n \in \B$ to obtain
\begin{align}
	&\beta \cdot \mbb{P}\left( \sup_{h \in \H} \sum_{i=1}^n U_{X_i,Y_i}(h) > x \right) \notag\\ 
	& \leq \mbb{P} \left(\sup_{h \in \H} n \cdot \left\{ 2(\emp - \emp')[\xi h] + 2 \En h^2- (1+c) \cdot \emp h^2 - (1 - c')\cdot \emp' h^2 \right\} > \frac{x}{2} \right) \label{eq:RHS1}
\end{align}
Next, we apply Assumption \ref{Assump:Low-Iso-Bd} with $\epsilon = c/4  = 1/72$ to terms in \eqref{eq:RHS1} to construct an offset Rademacher process. Note
$$
\frac{2}{1-\epsilon} < 2 ( 1+ 2 \epsilon)  = 2 + c.
$$
We can now choose $\tilde{c}, c'>0$ in  that satisfy
\begin{align}
	\frac{2}{1 - \epsilon} \leq 2+c-c'-2\tilde{c} ~~~~\Longleftrightarrow~~~~ 1 - (1-c'-\tilde{c})(1-\epsilon) \leq (1+c - \tilde{c})(1-\epsilon) -1.
\end{align}
Choose $b$ now such that
\begin{align}
\label{eq:b}
1 - (1-c'-\tilde{c})(1-\epsilon) \leq b \leq (1+c - \tilde{c})(1-\epsilon) -1.
\end{align}
Then we have on the set $\H$, applying lower isometry bound and Eq.~\eqref{eq:b}, with probability at least $1 - 2\delta$, 
\begin{align*}
	\emp (f - f^*)^2 \geq (1-\epsilon)  \cdot \En (f - f^*)^2 ~~\Longrightarrow~~ (1+b) \En h^2- (1+c) \cdot \emp h^2 \leq -\tilde{c} \cdot \emp h^2, \\
	\emp' (f - f^*)^2 \geq (1-\epsilon) \cdot \En (f - f^*)^2 ~~\Longrightarrow~~  (1-b) \En h^2 - (1 - c')\cdot \emp' h^2  \leq -\tilde{c} \cdot \emp' h^2.
\end{align*}
Thus we can continue bounding the expression in \eqref{eq:RHS1} as
\begin{align*}
	& \quad \sup_{h \in \H} n \cdot \left\{ 2(\emp - \emp')[\xi h] + 2 \En h^2- (1+c) \cdot \emp h^2 - (1 - c')\cdot \emp' h^2 \right\} \\
	& = \sup_{h \in \H  } n \cdot \left\{ 2(\emp - \emp')[\xi h] + (1+b) \En h^2- (1+c) \cdot \emp h^2 + (1-b) \En h^2 - (1 - c')\cdot \emp' h^2 \right\} \\
	& \leq \sup_{h \in \H } n \cdot \left\{ 2(\emp - \emp')[\xi h] -\tilde{c} \cdot \emp h^2  - \tilde{c} \cdot \emp' h^2 \right\} \\
\end{align*}
For the probability of deviation, we obtain
\begin{align*}
        & \quad \beta \cdot \mbb{P}\left( \sup_{h \in \H} \sum_{i=1}^n U_{X_i,Y_i}(h) > x \right) \\ 
	& \leq \mbb{P} \left(\sup_{h \in \H} n \cdot \left\{ 2(\emp - \emp')[\xi h] -\tilde{c} \cdot \emp h^2 - \tilde{c} \cdot \emp' h^2 \right\} > \frac{x}{2}  \right) + 2\delta \label{eq:plug-emr-sq-bd}\\
	& =  \mbb{P} \left(\sup_{h \in \H} n \cdot \left\{ 2(\emp - \emp')[\epsilon\xi h] - \tilde{c} \cdot \emp h^2 -  \tilde{c} \cdot \emp' h^2 \right\} > \frac{x}{2}  \right)+ 2\delta  \\
	& \leq 2 \mbb{P} \left( \sup_{h \in \H}  \left\{ \sum_{i=1}^n 2\epsilon_i \xi_i h(X_i) - \tilde{c} \cdot  \sum_{i=1}^n h(X_i)^2 \right\} > \frac{x}{4}  \right) + 2 \delta.
\end{align*}
%
	To estimate $\beta$, write
\begin{align}
	\beta & = \inf_{h \in \H} \mbb{P} \left( \sum_{i=1}^n  V_{X_i,Y_i}(h) < \frac{x}{2} \right) \\
	& =  1 - \sup _{h \in \H} \mbb{P} \left( \sum_{i=1}^n  2\xi_ih(X_i) - \mathbb{E} [2\xi h] -  \En h^2 + (1-c') \cdot h(X_i)^2 \geq \frac{x}{2} \right).
\end{align}

Let's bound the last term in above equation, for any $h \in \H$
\begin{align}
& \mbb{P}\left( (\emp - \En) [2\xi h]+(1-c')\emp h^2 - \En h^2 > \frac{x}{2n}  \right) \\
\leq & \mbb{P}\left( (\emp - \En)  [2\xi h] > \frac{x}{2n} + \frac{c'}{2} \En h^2 \right)+\mbb{P}\left( (\emp-\En) [h^2]  > \frac{c'}{2(1-c')} \En h^2 \right). \label{beta.eq} \\
\end{align}
Define
$$
A := \sup_{h \in \H} \frac{\En h^4}{(\En h^2)^2} ~~~\text{and}~~~ B:= \sup_{X, Y} \En \xi^4.
$$
Then for the second term in Eq~\eqref{beta.eq}, using Chebyshev's inequality
\begin{align*}
\mbb{P}\left( (\emp-\En) [h^2]  > \frac{c'}{2(1-c')} \En h^2 \right)  & \leq \frac{4(1-c')^2A}{c'^2n} \leq 1/4
\end{align*}
if $$n \geq \frac{16 (1-c')^2 A}{c'^2}.$$
For the first term in Eq~\eqref{beta.eq}, note
$$
{\sf Var}[2\xi h] \leq 4 \En [\xi^2 h^2] \leq 4 \sqrt{AB} \cdot \En h^2 
$$
and thus through Chebyshev inequality
\begin{align*}
\mbb{P}\left( (\emp - \En)  [2\xi h] > \frac{x}{2n} + \frac{c'}{2} \En h^2 \right) & \leq \frac{4 \sqrt{AB} \cdot \En h^2}{n\left( \frac{x}{2n} + \frac{c'}{2} \En h^2  \right)^2} \\
& \leq  \frac{4 \sqrt{AB} \cdot \En h^2}{n \cdot 4 \frac{x}{2n} \cdot \frac{c'}{2} \En h^2}  \leq \frac{1}{4}
\end{align*}
if 
$$
x \geq \frac{16\sqrt{AB}}{c'}.
$$
Assemble above bounds, for any $h \in \H$
$$
 \sup _{h \in \H} \mbb{P} \left( \sum_{i=1}^n  2\xi_ih(X_i) - \mathbb{E} [2\xi h] -  \En h^2 + (1-c') \cdot h(X_i)^2 \geq \frac{x}{2} \right)\leq \frac{1}{2}
$$
which further implies $\beta \geq 1/2$ for any $x> \frac{16\sqrt{AB}}{c'}$ and whenever
 $$n > \frac{16 (1-c')^2 A}{c'^2}.$$
 
Under the above regime,
\begin{align*}
	 &\frac{1}{2} \mbb{P}\left( \sup_{h \in \H} \sum_{i=1}^n U_{X_i,Y_i}(h)  > x \right)\leq 2 \mbb{P}\left( \sup_{h \in \H}  \left\{ \sum_{i=1}^n \epsilon_i \xi_i h(X_i) - \tilde{c} \cdot  \sum_{i=1}^n h(X_i)^2 \right\} > \frac{x}{4}  \right)+ 2 \delta
\end{align*}
and so
\begin{align*}
	&\quad \mbb{P}\left( \sup_{h \in \H} \sum_{i=1}^n U_{X_i,Y_i}(h)  > 4t \right)  \\
	&\leq 4 \mbb{P}\left( \sup_{h \in \H}  \left\{ \sum_{i=1}^n \epsilon_i \xi_i h(X_i) - \tilde{c} \cdot  \sum_{i=1}^n h(X_i)^2 \right\} >t  \right) + 4\delta.
\end{align*}
We conclude by writing
\begin{align*}
	&\mbb{P}\left( \sup_{h \in \H} (\emp-\En)[2\xi h] + \En h^2 - (1+c) \cdot \emp h^2 > 4t  \right) \\
	& \leq 4 \mbb{P} \left( \sup_{h \in \H} \frac{1}{n} \sum_{i=1}^n \epsilon_i \xi_i h(X_i) - \tilde{c} \cdot  \sum_{i=1}^n h(X_i)^2  > t  \right) + 4 \delta.
\end{align*}

\end{proof}

\begin{proof}[\textbf{Proof of Lemma~\ref{eq:finite}}]
	Using a standard argument,
	\begin{align*}
		\En_\epsilon \max_{v\in V} \left[ \sum_{i=1}^n \epsilon_i v_i - Cv_i^2 \right] &\leq \frac{1}{\lambda}\log\sum_{v\in V} \En_\epsilon \exp\left\{\sum_{i=1}^n \lambda \epsilon_i v_i - \lambda Cv_i^2\right\}.
	\end{align*}
	For any $v\in V$,
	\begin{align*}
		\En_\epsilon \exp\left\{\sum_{i=1}^n \lambda\epsilon_i v_i - \lambda Cv_i^2\right\} 
		\leq \exp\left\{\sum_{i=1}^n \lambda^2 v_i^2/2 - \lambda Cv_i^2\right\} \leq 1
	\end{align*}
	by setting $\lambda = 2C $. The first claim follows. For the second claim,
	\begin{align*}
		\Pr{\max_{v\in V} \left[ \sum_{i=1}^n \epsilon_i v_i - Cv_i^2 \right] \geq \frac{1}{2C} \log ( N/\delta)} 
		&\leq \En\exp\left\{ \lambda \max_{v\in V} \left[ \sum_{i=1}^n \epsilon_i v_i - Cv_i^2 \right] - \lambda\frac{1}{2C} \log( N/\delta) \right\}\\
		&\leq \sum_{v\in V} \En \exp\left\{ \lambda \left[ \sum_{i=1}^n \epsilon_i v_i - Cv_i^2 \right] - \lambda\frac{1}{2C} \log( N/\delta)  \right\}\\
		&\leq \sum_{v\in V} \exp\left\{  - \log (N/\delta) \right\} = \delta.
	\end{align*}

Now let's move to the case where $\xi$, the noise is unbounded.
\begin{align*}
	&\En_{\epsilon} \frac{1}{n} \max_{v\in V } \left\{ \sum_{i=1}^n \epsilon_i \xi_i v_i - C v_i^2 \right\} \leq \frac{1}{n\lambda} \log \En_{\epsilon} \sum_{v\in V} \exp \left( \lambda\sum_{i=1}^n \epsilon_i \xi_i v_i -\lambda C v_i^2    \right) \\
	&\leq \frac{1}{n\lambda} \log  \sum_{v\in V} \exp \left( \sum_{i=1}^n  \frac{\lambda^2}{2} \xi_i^2 v_i^2  -   \lambda C v_i^2   \right) \leq \max_{v \in V \setminus \{0\}} \frac{\sum_{i=1}^n v_i^2 \xi_i^2}{2C \sum_{i=1}^n v_i^2} \cdot \frac{\log N}{n}
\end{align*}
if we take $\lambda =   \min_{v \in V \setminus \{0\}} \frac{2C \sum_{i=1}^n v_i^2}{\sum_{i=1}^n v_i^2 \xi_i^2}$. The high probability statement follows also use this particular choice of $\lambda$.

\end{proof}

\begin{proof}[\textbf{Proof of Lemma~\ref{lem:offset_estimate_chaining}}]
	The proof proceeds as in \citep{RakSri14nonparam}. Fix $\gamma\in[0,1]$. By definition of a cover, there exists a set $V\subset \reals^n$ vectors of size $ N=\cN_2(\G,\gamma)$ with the following property: for any $g\in\G$, there exists a $v=v[g]\in V$ such that
	$$\frac{1}{n}\sum_{i=1}^n (g(z_i)-v_i)^2 \leq \gamma^2.$$
	Then we may write, 
	\begin{align}
		\label{eq:decomp}
		& \quad \En_\epsilon \sup_{g\in\G} \left[ \sum_{t=1}^n \epsilon_i g(z_i) - Cg(z_i)^2 \right] \\
		& \leq \En_\epsilon \sup_{g\in\G} \left[ \frac{1}{n}\sum_{t=1}^n \epsilon_i (g(z_i)-v[g]_i) \right]
		+ \En_\epsilon \sup_{g\in\G} \left[ \sum_{t=1}^n (C/4)v[g]_i^2- Cg(z_i)^2  \right]
		+ \En_\epsilon \sup_{g\in\G} \left[ \sum_{t=1}^n \epsilon_i v[g]_i - (C/4)v[g]_i^2 \right] 
	\end{align}
	We now argue that the second term is nonpositive. More precisely, we claim that for any $g\in\G$, 
	\begin{align}
		\label{eq:req_norm_comparison}
		\frac{1}{4}\sum_{t=1}^n v[g]_i^2 \leq \sum_{t=1}^n g(z_i)^2
	\end{align}
	for some element $v[g]\in V \cup \{\bf{0}\}$. First, consider the case $\sum_{t=1}^n g(z_i)^2\leq \gamma^2$. Then $v[g]=\bf{0}$ is an element $\gamma$-close to values of $g$ on the sample, and \eqref{eq:req_norm_comparison} is trivially satisfied. Next, consider the case $\sum_{t=1}^n g(z_i)^2 > \gamma^2$ and write $u=(g(z_1),\ldots,g(z_n))$. The triangle inequality for the Euclidean norm yields
	$$\|v[g]\| \leq \|v[g]-u\|+\|u\| \leq \gamma+\|u\| \leq 2\|u\|,$$
	establishing non-positivity of the second term in \eqref{eq:decomp}. 
	The third term in \eqref{eq:decomp} is upper bounded with the help of Lemma~\ref{eq:finite} as
	\begin{align*}
		\En_\epsilon \max_{g\in\G} \left[ \sum_{t=1}^n \epsilon_i v[g]_i - (C/4)v[g]_i^2 \right] \leq \frac{2}{C}\log \cN_2(\G,\gamma)
	\end{align*}
	Finally, the first term in \eqref{eq:decomp} is upper bounded using the standard chaining technique, keeping in mind that the $\ell_2$-diameter of the indexing set is at most $\gamma$.
\end{proof}

\begin{proof}[\textbf{Proof of Lemma~\ref{lma:prob-chaining}}]
	The proof is similar to the proof of Lemma~\ref{lem:offset_estimate_chaining}. We proceed with the following decomposition:
	\begin{align*}
		\sup_{g\in\G} \left[ \frac{1}{n}\sum_{t=1}^n \epsilon_i g(z_i) - Cg(z_i)^2 \right] \leq  \sup_{g\in\G} \left[ \frac{1}{n}\sum_{t=1}^n \epsilon_i (g(z_i)-v[g]_i) \right] +  \sup_{g\in\G} \left[ \frac{1}{n}\sum_{t=1}^n \epsilon_i v[g]_i - (C/4)v[g]_i^2 \right].
	\end{align*}
For the first term, we can employ the traditional high probability chaining bound. For some $c>0$, the following holds,
	\begin{align*}
		\mbb{P}_\epsilon \left( \sup_{g\in\G} \left[ \frac{1}{n}\sum_{t=1}^n \epsilon_i (g(z_i)-v[g]_i) \right] > u \cdot \inf_{\alpha \in [0,\gamma]} \left\{ 4\alpha + \frac{12}{\sqrt{n}} \int_{\alpha}^{\gamma} \sqrt{\log \mc{N}_2(\G,\delta)} d\delta \right\} \right) \leq  \frac{2}{1-e^{-2}} \exp(-cu^2).
	\end{align*}
For the second term, 
	\begin{align*}
		\mbb{P}_\epsilon \left(  \sup_{g\in\G} \left[ \frac{1}{n}\sum_{t=1}^n \epsilon_i v[g]_i - (C/4)v[g]_i^2 \right] >  \frac{2}{C} \frac{\log \mc{N}_2(\G,\gamma) + u}{n} \right) \leq \exp(-u).
	\end{align*}
	Combining the above two bounds, we have
	\begin{align*}
		&\quad \mbb{P}_\epsilon \left( \sup_{g\in\G} \left[ \frac{1}{n}\sum_{t=1}^n \epsilon_i g(z_i) - Cg(z_i)^2 \right] >  u \cdot \inf_{\alpha \in [0,\gamma]} \left\{ 4\alpha + \frac{12}{\sqrt{n}} \int_{\alpha}^{\gamma} \sqrt{\log \mc{N}_2(\G,\delta)} d\delta \right\}  +  \frac{2}{C} \frac{\log \mc{N}_2(\G,\gamma) + u}{n} \right) \\
		& \leq \mbb{P}_\epsilon \left( \sup_{g\in\G} \left[ \frac{1}{n}\sum_{t=1}^n \epsilon_i (g(z_i)-v[g]_i) \right] > u \cdot \inf_{\alpha \in [0,\gamma]} \left\{ 4\alpha + \frac{12}{\sqrt{n}} \int_{\alpha}^{\gamma} \sqrt{\log \mc{N}_2(\G,\delta)} d\delta \right\} \right) \\
		& \quad + \mbb{P}_\epsilon \left(  \sup_{g\in\G} \left[ \frac{1}{n}\sum_{t=1}^n \epsilon_i v[g]_i - (C/4)v[g]_i^2 \right] >  \frac{2}{C} \frac{\log \mc{N}_2(\G,\gamma) + u}{n} \right) \\
		& \leq \frac{2}{1-e^{-2}} \exp(-cu^2) + \exp(-u).
	\end{align*}
\end{proof}

\begin{proof}[\textbf{Proof of Theorem~\ref{thm:crit_radius}}]
Denote by $\B$ the unit ball with respect to $\ell_2$ distance, $\B = \{ h: (\En h^2)^{1/2} \leq 1 \}$, and let $\S$ denote the unit sphere.
Choosing any $h \in\cH \backslash r\B$, we have $\| h \|_{\ell_2} >r \deq \alpha_n(\H, \kappa',\delta)$ with $k'$ to be chosen later. Under the assumption that $\H$ is star-shaped, we know $h_r:=r/\| h \|_{\ell_2}\cdot h \in\cH$, thus
\begin{align*}
 &\frac{2}{n} \sum_{i=1}^n \epsilon_i \xi_i h(X_i) - c' \frac{1}{n} \sum_{i=1}^n h^2(X_i) \\
 =&  \frac{\| h \|_{\ell_2}}{r} \frac{2}{n} \sum_{i=1}^n \epsilon_i \xi_i h_r(X_i) - \left(\frac{\| h \|_{\ell_2}}{r}\right)^2  c' \frac{1}{n} \sum_{i=1}^n h_r^2(X_i) \\
 =& \frac{\| h \|_{\ell_2}}{r} \left\{ \frac{2}{n} \sum_{i=1}^n \epsilon_i \xi_i h_r(X_i) - c' \frac{1}{n} \sum_{i=1}^n h_r^2(X_i) \right\} - \frac{\| h \|_{\ell_2}}{r} \left(\frac{\| h \|_{\ell_2}}{r}-1\right) c' \frac{1}{n} \sum_{i=1}^n h_r^2(X_i).
\end{align*}
Comparing the supremum of the offset Rademacher process outside the ball $r\B$ with the one inside the ball $r\B$, we have
\begin{align}
	& \sup_{h\in\cH \backslash r\B} \left\{ \frac{2}{n} \sum_{i=1}^n \epsilon_i \xi_i h(X_i) - c' \frac{1}{n} \sum_{i=1}^n h^2(X_i) \right\} - \sup_{h\in\cH \cap r\B} \left\{ \frac{2}{n} \sum_{i=1}^n \epsilon_i \xi_i h(X_i) - c' \frac{1}{n} \sum_{i=1}^n h^2(X_i) \right\} \notag\\
	& \leq \sup_{h\in\cH \backslash r\B} \left\{  \left(\frac{\| h \|_{\ell_2}}{r}-1\right) \sup_{h_r \in \H \cap r\B} \left\{ \frac{2}{n} \sum_{i=1}^n \epsilon_i \xi_i h_r(X_i) - c' \frac{1}{n} \sum_{i=1}^n h_r^2(X_i) \right\} \right. \notag\\
	&\left.\hspace{1in} - \frac{\| h \|_{\ell_2}}{r} \left(\frac{\| h \|_{\ell_2}}{r}-1\right) \inf_{h_r \in \H \cap r\S}\left\{ c' \frac{1}{n} \sum_{i=1}^n h_r^2(X_i) \right\} \right\} \notag\\
	& \leq \sup_{h\in\cH \backslash r\B} \left\{  \left(\frac{\| h \|_{\ell_2}}{r}-1\right) \left\{ \sup_{h_r \in \H \cap r\B} \left\{ \frac{2}{n} \sum_{i=1}^n \epsilon_i \xi_i h_r(X_i) - c' \frac{1}{n} \sum_{i=1}^n h_r^2(X_i) \right\} - \inf_{h_r \in \H_r \cap r\S}\left\{ c' \frac{1}{n} \sum_{i=1}^n h_r^2(X_i) \right\} \right\}  \right\} \label{eq:cr.rad}.
\end{align}

If
$$
 \kappa' r^2 \leq c' (1-\epsilon) r^2,
$$
we can apply the lower isometry bound \ref{Assump:Low-Iso-Bd} and conclude
$$
\sup_{h\in\cH \cap r\B} \left\{ \frac{2}{n} \sum_{i=1}^n \epsilon_i \xi_i h(X_i) - c' \frac{1}{n} \sum_{i=1}^n h^2(X_i) \right\} \leq k'r^2\leq c' (1-\epsilon) r^2 \leq  \inf_{h_r \in \H \cap r\S}\left\{ c' \frac{1}{n} \sum_{i=1}^n h_r^2(X_i) \right\}
$$
with probability at least $1-2\delta$.

 Under this event, the difference of terms in \eqref{eq:cr.rad} is smaller than $0$, and we conclude
\begin{align*}
& \sup_{h\in\cH \backslash r\B} \left\{ \frac{2}{n} \sum_{i=1}^n \epsilon_i \xi_i h(X_i) - c' \frac{1}{n} \sum_{i=1}^n h^2(X_i) \right\} - \sup_{h\in\cH \cap r\B} \left\{ \frac{2}{n} \sum_{i=1}^n \epsilon_i \xi_i h(X_i) - c' \frac{1}{n} \sum_{i=1}^n h^2(X_i) \right\}\\
\leq  & \sup_{h\in\cH \backslash r\B} \left\{  \left(\frac{\| h \|_{\ell_2}}{r}-1\right) \left\{ \sup_{h_r \in \H \cap r\B} \left\{ \frac{2}{n} \sum_{i=1}^n \epsilon_i \xi_i h_r(X_i) - c' \frac{1}{n} \sum_{i=1}^n h_r^2(X_i) \right\} - \inf_{h_r \in \H_r \cap r\S}\left\{ c' \frac{1}{n} \sum_{i=1}^n h_r^2(X_i) \right\} \right\}  \right\} \\
\leq & \sup_{h\in\cH \backslash r\B} \left\{  \left(\frac{\| h \|_{\ell_2}}{r}-1\right)  \left( \kappa' r^2 - c' (1-\epsilon) r^2 \right) \right\} \leq 0
\end{align*}
Thus the excess loss is upper bounded by the offset Rademacher process, and the latter is further bounded by the process restricted within the critical radius:
\begin{align}
 	\sup_{h\in\cH} \left\{ \frac{2}{n} \sum_{i=1}^n \epsilon_i \xi_i h(X_i) - c' \frac{1}{n} \sum_{i=1}^n h^2(X_i) \right\} 
 	& \leq  \sup_{h\in\cH \cap r\B} \left\{ \frac{2}{n} \sum_{i=1}^n \epsilon_i \xi_i h(X_i) - c' \frac{1}{n} \sum_{i=1}^n h^2(X_i)\right\}\\
	& \leq  \alpha^2_n(\H, c' (1-\epsilon), \delta)
\end{align}
with probability at least $1 -2 \delta $.

%
\end{proof}

\begin{proof}[\textbf{Proof of Theorem~\ref{thm: mini-low-bd}}]
Denote $\F \subset \G = \F + \star(\F-\F)$. The minimax excess loss can be written as
\begin{align*}
	&\inf_{\hat{g} \in \G} \sup_{P}  \left\{ \En (\hat{g} - Y)^2 - \inf_{f \in \F} \En (f - Y)^2 \right\} \\
	& = \inf_{\hat{g} \in \G} \sup_{P}  \left\{   \left\{ - \En 2 Y \hat{g} + \En \hat{g}^2 \right\} + \sup_{f \in \F} \left\{ \En 2 Y f - \En f^2 \right\} \right\}.
\end{align*}
Now let's construct a particular distribution $P$ in the following way: take any $x_1, x_2,..., x_{(1+c)n} \in \X$ and let $P_X$ be the uniform distribution on these $(1+c)n$ points. For any $\epsilon = (\epsilon_1, \ldots, \epsilon_{(1+c)n}) \in \{ \pm 1\}^{(1+c)n}$, denote the distribution $P_{\epsilon}$ of $(X, Y)$ indexed by $\epsilon$ to be: $X$ is sampled from $P_X$, and $Y_{| X = x_i} = \epsilon_i,~\forall 1\leq i \leq (1+c)n$. Note here $\hat{g} : (X, Y)^{\otimes n} \rightarrow \F+\star(\F-\F)$. Now we proceed with this particular distribution
\begin{align*}
	&\inf_{\hat{g} \in \G} \sup_{P} \left\{   \left\{ - \En 2 Y \hat{g} + \En \hat{g}^2 \right\} + \sup_{f \in \F} \left\{ \En 2 Y f - \En f^2 \right\} \right\} \\
	&\geq \inf_{\hat{g} \in \G}  \sup_{\{x_i\}_{i=1}^{(1+c)n} \in \X^{\otimes (1+c)n}} \En_{\epsilon} \left\{   \left\{ - \En 2 Y \hat{g} + \En \hat{g}^2 \right\} + \sup_{f \in \F} \left\{ \En 2 Y f - \En f^2 \right\} \right\} \\
	& \geq   \sup_{\{x_i\}_{i=1}^{(1+c)n} \in \X^{\otimes (1+c)n}} \En_{\epsilon} \left\{  \sup_{f \in \F} \frac{1}{(1+c)n}\left\{\sum_{i=1}^{(1+c)n} 2  \epsilon_i f(x_i) - f(x_i)^2 \right\} \right\}  \\
	&\hspace{2in} - \sup_{\hat{g} \in \G}  \sup_{\{x_i\}_{i=1}^{(1+c)n} \in \X^{\otimes (1+c)n}} \En_{\epsilon} \left\{   2 \En Y \hat{g} - \En \hat{g}^2   \right\}.
\end{align*}
Note that the first term is exactly $\Rad^{\sf  o}((1+c)n, \F)$. Let us upper bound the second term. Denote the indices of a uniform $n$ samples from $(1+c)n$ samples $\{x_i\}_{i=1}^{(1+c)n}$ with replacement as $i_1,i_2,\ldots, i_n$, and $I$ be the set of unique indices $|I|\leq n$. Observe that $\hat{g}$ is a function of $(x_{I},Y_{I})$ only, independent of $\epsilon_{j}, j\notin I$.
\begin{align}
	&\sup_{\hat{g} \in \G}  \sup_{\{x_i\}_{i=1}^{(1+c)n} \in \X^{\otimes (1+c)n}} \En_{\epsilon} \left\{   2 \En Y \hat{g} - \En \hat{g}^2   \right\} \notag\\
	&  \leq \sup_{\hat{g} \in \G}  \sup_{\{x_i\}_{i=1}^{(1+c)n} \in \X^{\otimes (1+c)n}}  \En_{\epsilon} \En_{i_1,\ldots, i_n}   \left\{ \frac{1}{(1+c)n} \sum_{i=1}^{(1+c)n} \left\{ 2 \epsilon_i  \hat{g}(x_i) -  \hat{g}(x_i)^2  \right\}\right\} \notag\\
	& =  \sup_{\hat{g} \in \G}  \sup_{\{x_i\}_{i=1}^{(1+c)n} \in \X^{\otimes (1+c)n}}  \En_{i_1,\ldots, i_n} \En_{\epsilon} \left\{ \frac{1}{(1+c)n} \sum_{i=1}^{(1+c)n} \left\{ 2 \epsilon_i  \hat{g}(x_i) -  \hat{g}(x_i)^2  \right\}\right\}  \label{eq:long}
\end{align}
Conditionally on $i_1, i_2,...,i_n$, 
$$
 \frac{1}{(1+c)n}\sum_{i \notin I} \left\{ 2 \epsilon_i \hat{g}(x_i) -  \hat{g}(x_i)^2  \right\} = 0 -\frac{1}{(1+c)n}\sum_{i \notin I}   \hat{g}(x_i)^2  < 0 .
$$
Expression in \eqref{eq:long} is upper bounded by 
\begin{align*}
	&\sup_{\hat{g} \in \G}  \sup_{\{x_i\}_{i=1}^{(1+c)n} \in \X^{\otimes (1+c)n}}  \En_{i_1,\ldots, i_n} \En_{\epsilon}   \left\{ \frac{1}{(1+c)n} \sum_{i\in I} \left\{ 2 \epsilon_i \hat{g}(x_i) -  \hat{g}(x_i)^2  \right\}\right\}  \\
	& \leq \sup_{\hat{g} \in \G}   \En_{i_1,\ldots, i_n}   \sup_{\{x_i\}_{i=1}^{|I|} \in \X^{\otimes |I|}} \En_{\epsilon}  \left\{ \frac{1}{(1+c)n} \sum_{i\in I} \left\{ 2 \epsilon_i \hat{g}(x_i) -  \hat{g}(x_i)^2  \right\}\right\}   \\
	& \leq \sup_{\hat{g} \in \G}      \sup_{\{x_i\}_{i=1}^{n} \in \X^{\otimes n}} \En_{\epsilon}   \left\{ \frac{1}{(1+c)n} \sum_{i=1 }^{cn} \left\{ 2 \epsilon_i \hat{g}(x_i) - \hat{g}(x_i)^2  \right\}\right\} \\
	& \leq     \sup_{\{x_i\}_{i=1}^{n} \in \X^{\otimes n}} \En_{\epsilon}   \sup_{g \in \G}   \left\{ \frac{1}{(1+c)n} \sum_{i=1 }^{cn} \left\{ 2 \epsilon_i g(x_i) - g(x_i)^2  \right\}\right\} \\
	&= \frac{c}{1+c}\Rad^{\sf o}(cn, \G).
\end{align*}
Thus the claim holds.

\end{proof}

\begin{proof}[\textbf{Proof of Lemma~\ref{lem:finite_agg}}]
From Lemma~\ref{lma:star-cov}, we know for $\H = \F -f^* + \star(\F - \F)$,
\begin{align*}
\log \mc{N}_2(\H, 8 \epsilon) \leq \log \mc{N}_2(\F - f^*, 4\epsilon) + \log \mc{N}_2(\star(\F - \F), 4\epsilon) \leq \log \frac{2}{\epsilon} + 3\log \mc{N}_2(\F,  \epsilon).
\end{align*}
Consider the $\delta$-covering net of $\H$, where for any $h \in \H$, $v[h]$ is the closest point on the net.
\begin{align*}
& \frac{1}{n} \sup_{h \in \H} \left\{ \sum_{i=1}^n 2 \epsilon_i \xi_i h(X_i) - C h(X_i)^2 \right\}   \\
& \leq \frac{1}{n} \sup_{h \in \H} \left\{ \sum_{i=1}^n 2 \epsilon_i \xi_i (h(X_i) - v[h]) - C (h(X_i)^2  -v[h]^2)\right\}  + \frac{1}{n} \sup_{v \in \mc{N}_2(\H, \delta)} \left\{ \sum_{i=1}^n 2 \epsilon_i \xi_i v - C v^2 \right\} \\
& \leq 2  (  \sqrt{\sum_{i=1}^n \xi_i^2/n} + 2C )  \cdot \delta + \frac{1}{n} \sup_{v \in \mc{N}_2(\H, \delta)} \left\{ \sum_{i=1}^n 2 \epsilon_i \xi_i v - C v^2 \right\}.
\end{align*}
The second term is the offset Rademacher for a finite set of cardinality at most $\log (16/\delta) + 3 \log N$, thus applying Lemma~\ref{eq:finite}, 
\begin{align*}
\En_{\epsilon} \frac{1}{n} \sup_{h \in \H} \left\{ \sum_{i=1}^n 2 \epsilon_i \xi_i h(X_i) - C h(X_i)^2 \right\}  & \leq \inf_{\delta>0}\left\{  K  \cdot \delta + M \cdot \frac{3 \log N + \log (16/\delta)}{n} \right\} \\
 & \leq \tilde{C} \cdot \frac{\log (N \vee n)}{n}
\end{align*}
where $K:=2  (  \sqrt{\sum_{i=1}^n \xi_i^2/n} + 2C )$ and $M$ is defined in Equation~\eqref{eq:M}.
We also have the high probability bound via Lemma~\ref{eq:finite}:
\begin{align*}
\mbb{P}_{\epsilon} \left( \frac{1}{n} \sup_{h \in \H} \left\{ \sum_{i=1}^n 2 \epsilon_i \xi_i h(X_i) - C h(X_i)^2 \right\}  \leq   \tilde{C} \cdot \frac{\log (N\vee n) + u }{n} \right) \leq e^{-u}.
\end{align*}
\end{proof}

\section*{Acknowledgements} 
We thank Shahar Mendelson for many helpful discussions and for providing valuable feedback on this  paper.

\bibliographystyle{apalike}
\bibliography{paper}

\end{document}